\documentclass[letterpaper, 10 pt, conference]{ieeeconf}  

\IEEEoverridecommandlockouts                              

\usepackage{amsmath,amsfonts}

\usepackage[utf8]{inputenc}
\usepackage{float}
\usepackage{xcolor} 

\PassOptionsToPackage{ruled,vlined}{algorithm2e}
\usepackage{algorithm2e}
\usepackage{algpseudocode}

\SetKwComment{Comment}{\# }{}
\SetCommentSty{itshape}
\SetKwInput{KwRequire}{Require}

\usepackage{array}
\usepackage[caption=false,font=normalsize,labelfont=sf,textfont=sf]{subfig}
\usepackage{textcomp}
\usepackage{stfloats}
\usepackage{verbatim}
\usepackage{cite}
\usepackage{bm}
\usepackage{multirow}
\usepackage{wrapfig}
\usepackage{siunitx}
\usepackage{arydshln}  
\usepackage{graphicx} 
\usepackage{caption}
\usepackage{float}
\usepackage{url}
\usepackage{tabularx} 
\usepackage{booktabs} 
\newtheorem{theorem}{Theorem}
\newtheorem{lemma}{Lemma}

\newtheorem{remark}{Remark}
\newtheorem{assumption}{Assumption}

\title{\LARGE \bf
Unified Hierarchical MPC in Task Executing for Modular Manipulators across Diverse Morphologies
}

\author{Maolin Lei$^{1}$, Edoardo Romiti$^{1}$, Arturo Laurenzi$^{1}$, Cheng Zhou$^{2}$, Wanli Xing$^{3}$, \\  Liang Lu$^{1}$, and Nikos G. Tsagarakis$^{1}$
\thanks{*This work was supported by the European Union’s Horizon 2020 Research and Innovation
Program under Project CONCERT (Grant agreement No 101016007). }
\thanks{$^{1}$Humanoids and Human Centered Mechatronics Research Line, Istituto Italiano Di Tecnologia (IIT) Genoa, Italy. 
        {\tt\small name.surname@iit.it}}%
\thanks{$^{2}$ Tencent Robotics X Lab, ShenZhen, China
        }%
\thanks{$^{3}$ The University of Hong Kong, Hong Kong, China}
}

\begin{document}

\maketitle
\thispagestyle{empty}
\pagestyle{empty}

\begin{abstract}

This work proposes a unified Hierarchical Model Predictive Control (H-MPC) for modular manipulators across various morphologies, as the controller can adapt to different configurations to execute the given task without extensive parameter tuning in the controller. The H-MPC divides the control process into two levels: a high-level MPC and a low-level MPC. The high-level MPC predicts future states and provides trajectory information, while the low-level MPC refines control actions by updating the predictive model based on this high-level information. This hierarchical structure allows for the integration of kinematic constraints and ensures smooth joint-space trajectories, even near singular configurations. Moreover, the low-level MPC incorporates secondary linearization by leveraging predictive information from the high-level MPC, effectively capturing the second-order Taylor expansion information of the kinematic model while still maintaining a linearized model formulation. This approach not only preserves the simplicity of a linear control model but also enhances the accuracy of the kinematic representation, thereby improving overall control precision and reliability. To validate the effectiveness of the control policy, we conduct extensive evaluations across different manipulator morphologies and demonstrate the execution of pick-and-place tasks in real-world scenarios.

\end{abstract}

\section{INTRODUCTION}

Modular manipulators have various possible predefined modules, which were introduced~\cite{matsumaru1995design, zhang2006novel, yim2007modular, romiti2021toward}. These modules can be rapidly assembled into the desired morphology in just a few minutes, allowing for the adaptable configuration of the manipulator into various forms. Specifically, with $n$ different modules available, the total number of potential configurations grows factorially ($n!$). For instance, a system with 10 modules can generate 3,628,800 distinct configurations. The task of the modular manipulator is typically defined as a precise sequence of waypoints for the end-effector in task space~\cite{stilman2010global, berenson2011task, Zhou2024ValueRL}. 
To execute a diverse set of tasks across different manipulator platforms, it becomes essential to develop a unified control policy that can generalize across varying task specifications and hardware configurations. 



Model Predictive Control (MPC)  is a controller capable of receding horizon control by considering predictive future states to adjust the current input while incorporating constraints~\cite{qin1997overview, mayne2014model, kohler2018nonlinear}, thus improving the precision and adaptability of the control process. Thanks to its advantages, MPC has become a cornerstone in the field of robotics, particularly for controlling robotic manipulators \cite{faroni2018predictive, kohler2018nonlinear, wang2024hierarchical, lee2023real, tulbure2020closing, lei2022mpc}, due to its ability to predict future states, adaptively adjust control inputs, and rigorously incorporate constraints, making it highly effective for complex robotic tasks.

In manipulator control, MPC has been used to solve inverse kinematics (IK) for redundant robots~\cite{faroni2018predictive, kohler2018nonlinear}, incorporating kinematic constraints for precise task execution. To manage multiple tasks, weighted task hierarchies in MPC~\cite{bouyarmane2017weight, minniti2019whole} assign priority levels to different objectives, with higher weights for more critical tasks. To enhance robustness, learning-based MPC approaches~\cite{hewing2019cautious, nubert2020safe} approximate nonlinear models, improving adaptability while incorporating safety and dynamics. Methods addressing singularity management~\cite{wang2024hierarchical, lee2023real, tulbure2020closing} ensure smooth joint motion, even near challenging configurations. For deployment in dynamic environments, collision avoidance is integrated into MPC~\cite{lei2022mpc, gafur2021dynamic, gaertner2021collision, kramer2020model}, allowing predictive collision avoidance. Beyond kinematics, MPC has also been extended to optimize interaction forces between robots and their environment, as seen in work combining admittance control with MPC~\cite{wahrburg2016mpc, gold2022model}. Additionally, several studies have extended the use of controllers to execute specific tasks, such as pushing~\cite{bauza2018data}, dynamic transportation~\cite{zhou2022topp}, grasping~\cite{zhou2023differential}, and soft object manipulation~\cite{zhou2024bimanual}, demonstrating the benefits of tailored MPC implementations.

Designing a unified MPC for these manipulators is particularly challenging because the number of possible configurations grows factorially with the module count.
This exponential growth highlights the need to design adaptable MPC control strategies for modular manipulators, particularly those that can be quickly deployed across various morphologies. To address these challenges, this work employs the concept of Hierarchical Model Predictive Control (H-MPC) \cite{mansard2009directional} to manage the control of modular manipulators across different morphologies. Same to MPC\cite{faroni2018predictive, kohler2018nonlinear, lei2022mpc}, H-MPC incorporates kinematic constraints while also ensuring smooth joint-space motion, even near singular configurations. The H-MPC, however, differs from MPC in its two-level structure: a high-level MPC and a low-level MPC. The high-level MPC predicts the future state of the robot and provides this information to the low-level MPC, which then refines its control model and makes precise adjustments based on the high-level forecasts. Furthermore, a redundancy formalism is integrated into H-MPC to efficiently manage the degrees of freedom (DoFs) of modular manipulators, particularly for the non-redundant morphology. 

To facilitate real-time implementation on embedded robotic systems, the MPC formulation is often represented in a Quadratic Programming (QP) format, which can be solved efficiently. Typically, this is achieved by representing the relationship between control inputs and system states as a linear function. Our previous work~\cite{lei2022mpc} employed a linearized kinematic model to capture the relationship between joint states and the end-effector. 
In this work, we continue to employ a linearized kinematic model as the predictive model in both the high-level and low-level MPC to balance computational efficiency and real-time control requirements. The high-level MPC provides predictive state information, which is used to update the linearized formulation in the low-level MPC. Consequently, the linear predictive model in the low-level MPC implicitly captures second-order effects while retaining its linearized structure. This approach ensures that the low-level MPC maintains its QP form for efficient computation, while also enhancing the accuracy and reliability of the overall control strategy.

\textit{Contribution:} 
We developed an H-MPC controller specifically designed for modular manipulators, which adapts to various morphologies without the need for parameter tuning. This controller integrates kinematic constraints and effectively smooths trajectories near singular configurations to enhance precision. Moreover, the low-level MPC component of our H-MPC includes second-order dynamics in the state-space equations while maintaining a linear formulation, allowing for more accurate control of manipulator dynamics. Through extensive experiments and simulations, we demonstrated the effectiveness and advantages of our H-MPC.

\section{Preliminaries}

\subsection{Task Trajectory Generation} ~\label{task_traj_gen}
\label{sec:trajectory_generation}
The desired task of the manipulator is represented as a trajectory in the task (Cartesian) space, defined by a series of waypoints that the manipulator's end-effector must follow~\cite{stilman2010global, berenson2011task}. The trajectory generation process comprises two primary components: the position trajectory and the orientation trajectory. The position trajectory is generated first using polynomial interpolation, followed by the interpolation of the orientation trajectory based on the position trajectory results to ensure synchronized and smooth motion.

\textit{Position Trajectory}: 
The position trajectory is obtained by interpolating between waypoints using polynomial interpolation. This method ensures smooth transitions between waypoints while adhering to constraints on velocity, acceleration, and workspace boundaries.
\begin{equation}
\bm{p}(t) = \sum_{i=0}^{n} \bm{a}_i t^i,
\label{eq:position_trajectory}
\end{equation}
where,
 \(\bm{p}(t)\) is the position at time \(t\), \(\bm{a}_i\) are the coefficients determined to satisfy the waypoint positions and boundary conditions; \(n\) is the degree of the polynomial.

\textit{Orientation Trajectory}: 
The orientation trajectory ensures a smooth and continuous transition from the initial orientation \(\mathbf{R}_{\text{in}}\) to the desired orientation \(\mathbf{R}_d\) using exponential mapping. This method operates within the \(SO(3)\) group, providing smooth interpolation between orientations.
\begin{equation}
\mathbf{R}(t) = \mathbf{R}_{\text{in}} \exp\left( \frac{t}{T} \log\left( \mathbf{R}_{\text{in}}^{-1} \mathbf{R}_\text{d} \right) \right),
\label{eq:orientation_trajectory}
\end{equation}
where \(\mathbf{R}(t)\) is the orientation at time \(t\); \(t \in [0, T]\); \(T\) is the total duration; \(\exp(\cdot)\) and \(\log(\cdot)\) denote the matrix exponential and logarithm, respectively. The total \(T\) can be obtained once the position trajectory generate. 

\subsection{Model}
The internal state \(\bm{x}_e \in \mathbb{R}^{13}\) and its derivation is given by:
\[
\bm{x}_e := [\bm{p}\ \ \bm{o} \ \ \bm{\dot{p}} \ \ \bm{\omega}]^T \in \mathbb{R}^{13}., \ \
\bm{\dot{x}}_e := [\bm{\dot{p}}\ \ \bm{\dot{o}} \ \ \bm{\dot{v}} \ \ \bm{\dot{\omega}}]^T \in \mathbb{R}^{12}.
\]
where \(\bm{p} \in \mathbb{R}^{3}\) is position, \(\bm{o} \in \mathbb{R}^{4}\) is the orientation quaternion, \(\bm{\dot{p}} \in \mathbb{R}^{3}\) is linear velocity, and \(\bm{\omega} \in \mathbb{R}^{3}\) is angular velocity for the end-effectror state. 
The quaternion derivative is
\(
\dot{\bm{o}}(t) = \frac{1}{2} \mathbf{G}(\bm{o}_k) \bm{\omega}_k \)
where
\(
\mathbf{G}(\bm{o}) = \begin{bmatrix}
-\bm{\epsilon}^\top, 
\eta \bm{I}_3 + \hat{\bm{\epsilon}}
\end{bmatrix}^{T},
\)
with \(\hat{\bm{\epsilon}}\) being the skew-symmetric matrix of \(\bm{\epsilon}\), and \(\eta\) the scalar part of the quaternion.


The primary objective of the controller is to accurately track the end-effector's reference trajectory by managing the manipulator's joints. 
The core of this process involves a first-order Taylor expansion of the manipulator's nonlinear kinematics around the current joint positions and velocities, establishing a relationship between the joint states and the end-effector's velocities and accelerations \cite{sciavicco2012modelling}. The derivation process is articulated as follows:

\begin{equation}
\underbrace{
\begin{bmatrix}
\bm{v}_e \\
\dot{\bm{v}}_e
\end{bmatrix}
}_{\substack{\text{End-effector kinematics} \\
\dot{\bm{x}}_e \in \mathbb{R}^{12}}}
=
\underbrace{
    \begin{bmatrix}
        \mathbf{J} (\bm{q})  & \bm{0}_{6 \times n_j} \\
        \dot{\mathbf{J}} (\bm{q},\bm{\dot{q}}) & \mathbf{J} (\bm{q}) 
    \end{bmatrix}
}_{\substack{\text{Kinematic mapping matrix} \\
\mathbf{B}_{\text{kin}} (\bm{q},\bm{\dot{q}}) \in \mathbb{R}^{12 \times 2n_j}}
}\ \
\underbrace{
\begin{bmatrix}
\dot{\bm{q}} \\
\ddot{\bm{q}}
\end{bmatrix}
}_{\substack{\text{Joint velocity and accerelation} \\ \bm{u}_{\text{kin}} \in \mathbb{R}^{2n_j}}},
\label{eq:kinematics_model}
\end{equation}
where \(\bm{v}_e \in \mathbb{R}^6\) consists of both the linear and angular velocities of the end-effector, and \(\dot{\bm{v}}_e \in \mathbb{R}^6\) represents the corresponding linear and angular accelerations.  \(\mathbf{J}(\bm{q}) \in \mathbb{R}^{6 \times n_j}\) represents the Jacobian of the manipulator in \( \bm{q}\) state  while \(\dot{\mathbf{J}} (\bm{q}, \bm{\dot{q}}) \in \mathbb{R}^{6 \times n_j}\) represent the time derivative of the \(\mathbf{J}\) in \(\bm{\dot{q}}\) and \(\bm{q}\) state.

\section{H-MPC Formulation} \label{MPC_formulation}
\subsection{MPC Formulation}

Combining the kinematic model in Eq.~\ref{eq:kinematics_model} with forward Euler integration using a step size of \(dt\), the discretized state-space equation for predicting the end-effector state \(\bm{x}_{e,k+1}\) at time step \(k+1\) from the state \(\bm{x}_{e,k}\) at time step \(k\) is given by:
\begin{equation}
    \bm{x}_{e,k+1} = \bm{x}_{e,k} + \underbrace{\mathbf{B}_{\text{e,k}} \mathbf{B}_{\text{kin,k}}(\bm{q}_k,\bm{\dot{q}}_k) \bm{u}_{\text{kin},k}}_{\delta \bm{x}_k  \approx \mathbf{B}_{\text{kin,k}}(\bm{q}_k,\bm{\dot{q}}_k) \delta \bm{u}_{\text{kin},k}} 
    \label{eq:MPC_kinematicals_discrete}
\end{equation}
where 
\(
\bm{u}_{\text{kin},k} = \begin{bmatrix} \dot{\bm{q}}_k \ \ \ddot{\bm{q}}_k \end{bmatrix} ^T\in \mathbb{R}^{2n_j},
\)
and
\(
\mathbf{B}_{\text{e,k}} = \operatorname{diag} \left[ \mathbf{I} \, dt, \, \frac{1}{2} \mathbf{G}(\bm{o}_k) \, dt, \, \mathbf{I} \, dt, \, \mathbf{I} \, dt \right] \in \mathbb{R}^{12 \times 12},
\)
with \(\mathbf{I}\) denoting the identity matrix. In MPC, the objective is to determine the optimal sequence of control inputs \(\{\bm{u}_{\text{kin},k}\}_{k=1}^N\) that minimizes the cumulative tracking error of the end-effector's trajectory over a prediction horizon \(N\). This optimization problem is formulated as:
\begin{align} 
    \min_{\{\bm{u}_{\text{kin},k}\}} \quad & \sum_{k=1}^{N} \left( \| \bm{x}_{e,k+1} - \bm{x}_{e,\text{ref},k+1} \|_{\textbf{Q}_k}^2 + \| \bm{u}_{\text{kin},k} \|_{\textbf{R}_k}^2 \right) \label{eq:MPC_cost} \\
    \text{s.t.} \quad 
    & \bm{x}_{e,k+1} = \bm{x}_{e,k} + \mathbf{B}_{\text{e},k} \mathbf{B}_{\text{kin},k}(\bm{q}_k, \dot{\bm{q}}_k) \\ 
    & \bm{q}_l \leq \mathbf{E} \, dt \, \dot{\bm{q}}_k + \bm{q}_0 \leq \bm{q}_u \\
    & \dot{\bm{q}}_{\min} \leq \dot{\bm{q}}_k \leq \dot{\bm{q}}_{\max}, \\
    & \ddot{\bm{q}}_{\min} \leq \ddot{\bm{q}}_k \leq \ddot{\bm{q}}_{\max}~\label{eq:joint_acc_con}, \\
     & \forall k = 1, \ldots, N,
\end{align}
where \(\textbf{Q}_k \in \mathbb{R}^{12 \times 12}\) and \(\textbf{R}_k \in \mathbb{R}^{2n_j \times 2n_j}\) are positive semi-definite weight matrices that penalize state errors and control inputs, respectively. The constraints include the limitation on joint positions, velocities, and accelerations to prevent the manipulator. In the MPC, The manipulator's kinematic model is linearized around the current kinematic state, and this linearized model remains valid throughout the entire prediction horizon \(N\). The regularization term \(\|\bm{u}_{\text{kin},k}\|_{\textbf{R}_k}\) penalizes large joint velocities and accelerations to avoid oscillatory joint-space motions near singularities where the inverse of \(\mathbf{B}_{\text{kin,k}}\) approaches infinity. 

\subsection{High-level MPC and Low-level MPC in H-MPC}
The task \(\mathcal{T}\) can be decomposed into six sub-tasks, corresponding to tracking the desired position and orientation along the \(x\), \(y\), and \(z\) axes, with \(\operatorname{dim}(\mathcal{T}) = 6\). For a redundant morphology—characterized by having more DoFs, denoted as \(\operatorname{dim}(\mathcal{R})\), than required by the task space (\(\operatorname{dim}(\mathcal{R}) > \operatorname{dim}(\mathcal{T})\) )—it is feasible to deploy the MPC directly, as the manipulator has sufficient freedom to execute directionally distinct sub-tasks without conflicts. Conversely, for non-redundant morphologies, where manipulators possess fewer DoFs than the dimensions of the task space, (\(\operatorname{dim}(\mathcal{R}) < \operatorname{dim}(\mathcal{T})\)), their ability to simultaneously execute sub-tasks in six directions is limited~\cite{gupta1986nature}. Employing the same MPC controller designed for redundant systems on non-redundant robots may lead to conflicts when executing directionally distinct tasks. To address this limitation, non-redundant manipulators can still be controlled using a redundancy-based approach known as \textit{functional redundancy} \cite{mansard2009directional}, which involves prioritizing sub-tasks into high-priority and secondary-priority categorie with the assumption that \textit{the number of high-priority sub-tasks and secondary-priority sub-tasks, denoted as \(\operatorname{dim}(\mathcal{T}_H)\) and \(\operatorname{dim}(\mathcal{T}_L)\), are both less than the manipulator's degrees of freedom, i.e., 
\(
\operatorname{dim}(\mathcal{T}_H) < \operatorname{dim}(\mathcal{R})
\) and
\(
\operatorname{dim}(\mathcal{T}_L) < \operatorname{dim}(\mathcal{R}).
\)} Building on redundancy control formulations, we apply the H-MPC to manage manipulators with varying morphologies using a redundancy-based control strategy. The H-MPC divides the control process into two levels: a high-level MPC and a low-level MPC. In the high-level MPC, the controller focuses solely on high-priority sub-tasks \(\mathcal{T}_H\), ensuring it operates within the available degrees of freedom. Under the assumption that \(
\operatorname{dim}(\mathcal{T}_H) < \operatorname{dim}(\mathcal{R}),
\) the high-level MPC functions effectively as a redundancy controller. At the low-level MPC, the controller incorporates information from the high-level MPC to refine its linearized model using the predicted state provided by the high-level MPC.

\begin{figure*}[h]
  \centering
\includegraphics[width=0.90\textwidth]{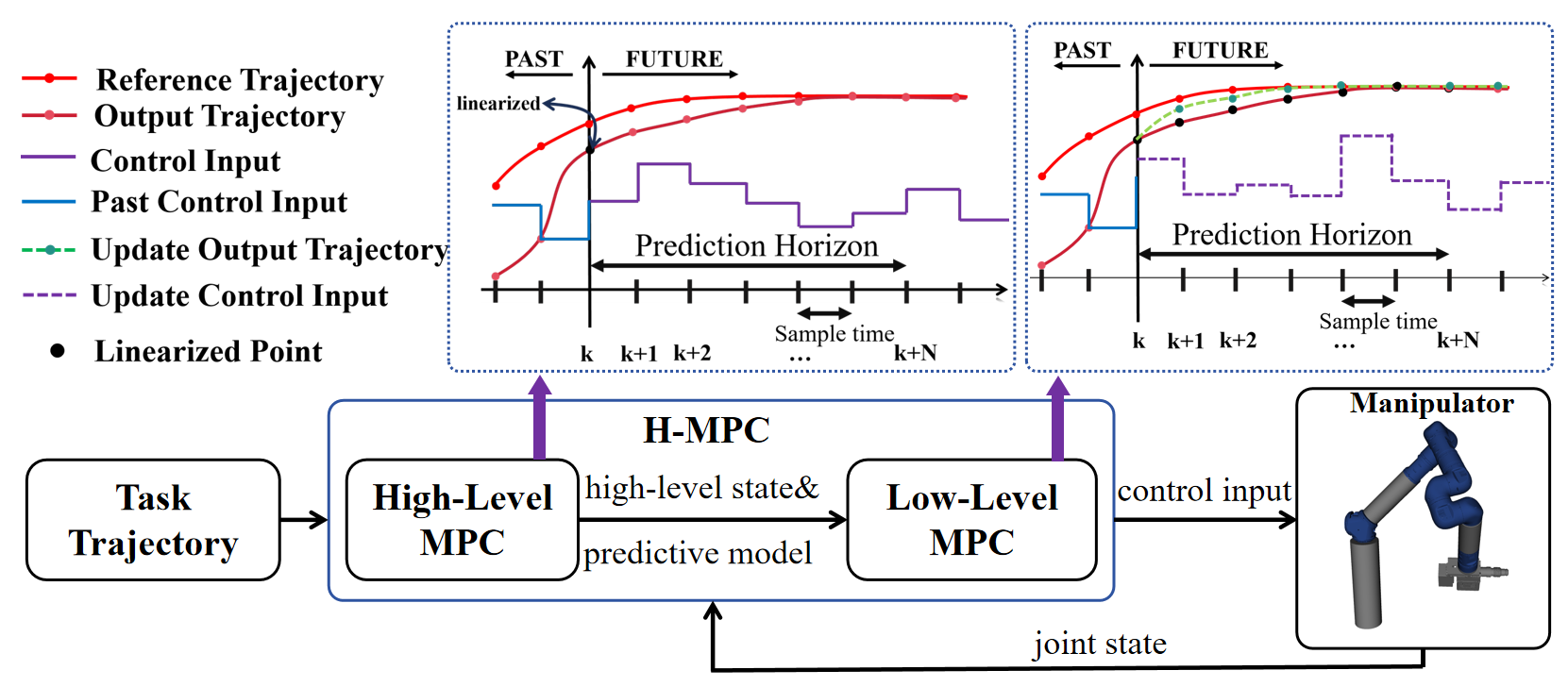}
 \vspace{-0.25cm}
  \caption{The concept of H-MPC and the overall control scheme for controlling the manipulator. In the high-level MPC, the model is linearized at the current state (represented by the black dot in the top middle of the figure). The control inputs at each step are then calculated, resulting in the trajectory shown as a purple line in the figure. In the low-level MPC, the model is linearized along the trajectory output by the high-level MPC (depicted by the black dots in the top right figure). The low-level MPC computes new control inputs based on this updated linearized model, producing the trajectory indicated by the purple dotted line in the top right figure.
  }
  \label{Fig:HMPC}
  \vspace{-0.65cm}
\end{figure*}

\subsubsection{High-Level MPC Formulation}
The high-priority MPC focus exclusively on the high-priority tasks without considering secondary tasks. The decision variable \( \bm{u}_{\text{kin},k}^{(1)} \) is the high-priority control input, and the optimization problem is formulated as follows:
\begin{equation}
\begin{aligned}
    \min_{\bm{u}_{\text{kin}}^{(1)}} \quad & \sum_{k=1}^{N_h} \left( \left\| \bm{x}_{e,k+1}^{(1)} - \bm{x}_{e,\text{ref},k+1}^{(1)} \right\|_{\bm{Q}_k^{(1)}}^2 + \left\| \bm{u}_{\text{kin},k}^{(1)} \right\|_{\mathbf{R}_k}^2 \right) 
    \\
    \text{s.t.}:  & \bm{x}_{e,k+1}^{(1)} = \bm{x}_{e,k}^{(1)} + \mathbf{B}_{\text{e},k} \mathbf{B}_{\text{kin},k}^{(1)} (\bm{q}_1,\bm{\dot{q}}_1) \bm{u}_{\text{kin},k}, \\ 
     & \text{kinematical constraints} \\ 
     & \forall k=1,\ldots,N_h
\end{aligned}
\label{eq:highlevel_MPC}
\end{equation}
where the superscript \( (1) \) indicates association with the high-priority tasks in the high-level layer of the hierarchical MPC; \(N_h\) represents the horizon of the high-level MPC; the gains for performing sub-tasks \( \bm{Q}_k^{(1)} \)
defines the weighting matrix for the high-priority tasks. Here, \( \bm{Q}_k \) represents the base weighting matrix, and \( \bm{P} \) is a binary vector, where each element specifies the inclusion (1) or exclusion (0) of the corresponding task direction in the high-priority MPC. The operation \( \text{diag}(\bm{P}) \) constructs a diagonal matrix from \( \bm{P} \), thus selectively weighting the tasks to prioritize them appropriately in the optimization problem.  The formulation high-level MPC as same as the MPC formulation with assume that the manipulator's kinematic model remains constant throughout the entire prediction horizon while the predictive state-space function is linearized in initial joint state of the manipulator. The high-level MPC formulation can final formulate it as the optimization problem to the QP formulation.



\subsubsection{Low-Level MPC Formulation}
With established priorities for various sub-tasks, the high-level MPC is designed to determine the necessary joint velocities and accelerations to execute the high-priority tasks effectively. Then, the low-level MPC is able to written as: 
\begin{align}
\min_{\bm{u}_{\text{kin}}^{(2)}} \quad 
& \sum_{k=1}^{N_l} \left( 
    \left\| \bm{x}_{e,k+1}^{(2)} - \bm{x}_{e,\text{ref},k+1}^{(2)} \right\|_{\mathbf{Q}_k^{(2)}}^2 
    + \left\| \bm{u}_{\text{kin},k}^{(2)} \right\|_{\mathbf{R}_k}^2 
\right) \nonumber \\
\text{s.t.} \quad 
& \bm{x}_{e,k+1}^{(2)} = \bm{x}_{e,k}^{(2)} + \mathbf{B}_{\text{e},k}^{(2)} 
\mathbf{B}_{\text{kin},k}^{(2)}(\bm{q}_{k}^{(1)}, \bm{\dot{q}}_{k}^{(1)}) 
\bm{u}_{\text{kin},k}^{(2)}, \label{eq:low_level_MPC_dyn} \\
& \bm{\dot{q}}_{e,k+1}^{(1)} = \bm{\dot{q}}_{e,k}^{(1)} + 
\bm{\ddot{q}}_{e,k}^{(1)} \Delta t, \label{eq:low_level_MPC_qdot} \\
& \bm{q}_{e,k+1}^{(1)} = \bm{q}_{e,k}^{(1)} + 
\bm{\dot{q}}_{e,k}^{(1)} \Delta t + 
0.5 \bm{\ddot{q}}_{e,k}^{(1)} \Delta t^2, \label{eq:low_level_MPC_q} \\
& \mathbf{B}_{\text{kin},k}^{(2)}(\bm{q}_{k}^{(1)}, \bm{\dot{q}}_{k}^{(1)}) 
\bm{u}_{\text{kin},k}^{(2)} = 
\mathbf{B}_{\text{kin},k}^{(1)}(\bm{q}_{1}^{(1)}, \bm{\dot{q}}_{1}^{(1)}) 
\label{eq:low_level_MPC}
\end{align}
where the superscript \( (2) \) denotes that the variable or parameter is associated with the secondary tasks in the second layer of the hierarchical MPC and \(N_l\) represents the horizon of the low-level MPC. In the low-level MPC, the high-level MPC's optimal results will serve as a constrained state input for the low-level MPC across each prediction horizon. This integration ensures that the low-level MPC can optimize joint states while taking into account the high-level control outcomes. Moreover, as illustrated in Fig.~\ref{Fig:HMPC} top right, the low-level MPC dynamically linearized the kinematic model at each timestep based on the high-level output. These iterative model updates enhance the accuracy of the predictive model by incorporating second-order Taylor expansions of the kinematic relationships between the end-effector and the joint states. Consequently, although the state-space representation of the low-level MPC remains linear, the precision of the predictions is significantly improved. This approach allows the low-level MPC to handle higher-order nonlinearities in the model while preserving the computational efficiency of linear models. We will provide a formal proof of this methodology in follow content. 
The pseudo algorithm of the HMPC can obtain in Alg.~\ref{alg:H_MPC_overall}.

\begin{algorithm}
\caption{Hierarchical MPC }
\label{alg:H_MPC_overall}
\KwRequire{
    Initial end-effector $\bm{x}_{e,0}$ and joint states $ \bm{q}_0$; \\
    Reference trajectory $\bm{p}_{\text{ref}}(t), \mathbf{R}_{\text{ref}}(t)$; \\
    Prediction horizons $N_h, N_l$; \\
    Control time step $\Delta t$; \\
    Total control iterations $n = T / \Delta t$.
}

\For{$k \gets 0$ \textbf{to} $n-1$}{
    Get end-effector state $\bm{x}_{e,k}$ based on $\bm{q}_k$ and $\dot{\bm{q}}_k$\;
    
    Obtain kinematic mapping $\mathbf{B}_{\text{kin},k}$ based on $\bm{q}_k$ and $\dot{\bm{q}}_k$\;
    
    Solve high-level MPC in (Eq.~\ref{eq:highlevel_MPC}) to obtain $\bm{u}_{\text{kin},k}^{(1)}$\;
    
    Update joint states  $\bm{q}_{k+i}^{(1)}$, $\dot{\bm{q}}_{k+i}^{(1)}$ base on $\bm{u}_{\text{kin},k}^{(1)}$\;
    
    \For{$i \gets 1$ \textbf{to} $N_l-1$}{
        Update $\mathbf{B}_{\text{kin},k+i}$ based on $\bm{q}_{k+i}^{(1)}$, $\dot{\bm{q}}_{k+i}^{(1)}$\;
    }

    Solve low-level MPC (Eq.~\ref{eq:low_level_MPC}) to obtain $\bm{u}_{\text{kin},k}^{(2)}$\;
    
    Apply Control Input $\bm{u}_{\text{kin},k}^{(2)}$\;

    Update $\bm{q}_k$ and $\dot{\bm{q}}_k$\;
}
\end{algorithm}
\textit{Proof for Implicit Inclusion of Second-Order Terms in Velocity Linearized Mapping Function:}

In the low-level MPC, the Jacobian at each prediction step is updated using the predicted joint positions from the high-level MPC:
\begin{equation}
\mathbf{J}_{k+i} = \mathbf{J}\left(\bm{q}_{k+i}^{(1)}\right), \quad \text{for } i = 0, \dots, N_l - 1,
\end{equation}
where $\bm{q}_{k+i}^{(1)}$ are the predicted joint positions from the high-level MPC. Our goal is to prove that updating the Jacobian in this manner implicitly includes second-order terms from the Taylor expansion of the end-effector position.

We begin by performing a Taylor expansion of $\mathbf{J}_{k+i}$ around $\bm{q}_k$:
\begin{equation}
\mathbf{J}_{k+i} \approx \mathbf{J}_k + \left. \dfrac{\partial \mathbf{J}}{\partial \bm{q}} \right|_{\bm{q}_k} \Delta \bm{q}_{k+i}^{(1)} + \mathcal{O}\left(\|\Delta \bm{q}_{k+i}^{(1)}\|^2\right),
\label{eq:jacobian_expansion}
\end{equation}
where $\Delta \bm{q}_{k+i}^{(1)} = \bm{q}_{k+i}^{(1)} - \bm{q}_k$, and $\dfrac{\partial \mathbf{J}}{\partial \bm{q}}$ is a third-order tensor representing the derivative of the Jacobian with respect to the joint positions.

The low-level MPC uses the following model for predicting the end-effector position, with the assumption that $\Delta {\bm{q}}_{k+i} \approx \dot{\bm{q}}_{k+i} \Delta t$:
\begin{equation}
\bm{x}_{e,k+i+1}^{(2)} = \bm{x}_{e,k+i}^{(2)} + \mathbf{J}_{k+i} \Delta \bm{q}_{k+i}^{(2)},
\label{eq:second_mpc_model}
\end{equation}
where $\Delta \bm{q}_{k+i}^{(2)} = \bm{q}_{k+i+1}^{(2)} - \bm{q}_{k+i}^{(2)}$.

In the hierarchical MPC, the high-level MPC computes the primary joint position changes $\Delta \bm{q}_{k+i}^{(1)}$ to achieve high-priority tasks. The low-level MPC then calculates adjustments $\Delta \bm{q}_{k+i}^{(2)}$ to address secondary objectives while minimally affecting the primary tasks.
To relate the terms in the updated Jacobian in the low-level MPC and the real Taylor expansion of the second-order term, we make the following assumption.

\begin{assumption}
    \textit{The total joint position change $\Delta \bm{q}_{k+i+1}$ can be approximated by the sum of the predicted changes from the high- and low-level MPCs:}
    \begin{equation}
    \Delta \bm{q}_{k+i+1} \approx \Delta \bm{q}_{k+i}^{(1)} + \Delta \bm{q}_{k+i}^{(2)}.
    \label{eq:assumption1}
    \end{equation}
\end{assumption}

Such a decomposition is critical for the subsequent error analysis and demonstrates the implicit inclusion of second-order terms in the low-level MPC model.
Therefore, the combined joint position at time $k+i+1$ is:
\begin{equation}
\bm{q}_{k+i+1} = \bm{q}_k + \Delta \bm{q}_{k+i+1} = \bm{q}_k + \Delta \bm{q}_{k+i}^{(1)} + \Delta \bm{q}_{k+i}^{(2)}.
\end{equation}

Considering the Taylor expansion up to the second-order term of the end-effector's position $\bm{x}_e$ around $\bm{q}_k$:
\begin{align}
\bm{x}_{e,k+i+1} =\, & \bm{x}_{e,k} + \mathbf{J}_k \Delta \bm{q}_{k+i+1} \nonumber + \underbrace{\dfrac{1}{2} \Delta \bm{q}_{k+i+1}^\top \mathbf{H}_k \Delta \bm{q}_{k+i+1}}_{\text{Second-order nonlinear term}} \\
& 
+ \mathcal{O}\left( \| \Delta \bm{q}_{k+i+1} \|^3 \right),
\label{eq:taylor_expansion}
\end{align}

where $\bm{H}_k = \left. \dfrac{\partial^2 \bm{x}_e}{\partial \bm{q}^2} \right|_{\bm{q}_k}$ is the Hessian matrix.

Substituting $\Delta \bm{q}_{k+i+1}$ from Assumption~\ref{eq:assumption1} into the Taylor expansion, we can compare the terms in the low-level MPC model \eqref{eq:second_mpc_model} and the Taylor expansion \eqref{eq:taylor_expansion}.

Substituting the expanded Jacobian from \eqref{eq:jacobian_expansion} into the low-level MPC model \eqref{eq:second_mpc_model}, we obtain:
\begin{align}
\bm{x}_{e,k+i+1}^{(2)} \approx\, & \bm{x}_{e,k+i}^{(2)} + 
\left[ \mathbf{J}_k + \left. \dfrac{\partial \mathbf{J}}{\partial \bm{q}} 
\right|_{\bm{q}_k} \Delta \bm{q}_{k+i}^{(1)} \right] 
\Delta \bm{q}_{k+i}^{(2)} \nonumber \\
=\, & \bm{x}_{e,k+i}^{(2)} + \mathbf{J}_k \Delta \bm{q}_{k+i}^{(2)} 
+ \left( \left. \dfrac{\partial \mathbf{J}}{\partial \bm{q}} 
\right|_{\bm{q}_k} \Delta \bm{q}_{k+i}^{(1)} \right) 
\Delta \bm{q}_{k+i}^{(2)} \nonumber \\
& + \mathcal{O}\left( \| \Delta \bm{q} \|^3 \right),
\label{eq:model_with_updated_jacobian}
\end{align}

Comparing equations \eqref{eq:taylor_expansion} and \eqref{eq:model_with_updated_jacobian}, we see that the term $\left( \left. \dfrac{\partial \mathbf{J}}{\partial \bm{q}} \right|_{\bm{q}_k} \Delta \bm{q}_{k+i}^{(1)} \right) \Delta \bm{q}_{k+i}^{(2)}$ in the low-level MPC model corresponds to part of the second-order term $\dfrac{1}{2} \Delta \bm{q}_{k+i+1}^\top \bm{H}_k \Delta \bm{q}_{k+i+1}$ in the Taylor expansion.

To analyze the error between the two models, we define the error term $\bm{E}_{k+i}$ as the difference between the second-order term in the Taylor expansion and the corresponding term in the low-level MPC model:
\begin{equation}
\bm{E}_{k+i} = \dfrac{1}{2} \Delta \bm{q}_{k+i+1}^\top \bm{H}_k \Delta \bm{q}_{k+i+1} - \left( \left. \dfrac{\partial \mathbf{J}}{\partial \bm{q}} \right|_{\bm{q}_k} \Delta \bm{q}_{k+i}^{(1)} \right) \Delta \bm{q}_{k+i}^{(2)}.
\end{equation}

Substituting $\Delta \bm{q}_{k+i+1} = \Delta \bm{q}_{k+i}^{(1)} + \Delta \bm{q}_{k+i}^{(2)}$, we expand the error term:
\begin{align}
\bm{E}_{k+i} =\, & \dfrac{1}{2} \left( \Delta \bm{q}_{k+i}^{(1)} 
+ \Delta \bm{q}_{k+i}^{(2)} \right)^\top \bm{H}_k 
\left( \Delta \bm{q}_{k+i}^{(1)} + \Delta \bm{q}_{k+i}^{(2)} \right) \nonumber \\
& - \left( \left. \dfrac{\partial \mathbf{J}}{\partial \bm{q}} 
\right|_{\bm{q}_k} \Delta \bm{q}_{k+i}^{(1)} \right) 
\Delta \bm{q}_{k+i}^{(2)} \nonumber \\
=\, & \dfrac{1}{2} \Big[ 
\Delta \bm{q}_{k+i}^{(1)\top} \bm{H}_k \Delta \bm{q}_{k+i}^{(1)} 
+ 2 \Delta \bm{q}_{k+i}^{(1)\top} \bm{H}_k \Delta \bm{q}_{k+i}^{(2)} \nonumber \\
& \quad + \Delta \bm{q}_{k+i}^{(2)\top} \bm{H}_k \Delta \bm{q}_{k+i}^{(2)} 
\Big] 
- \Delta \bm{q}_{k+i}^{(1)\top} \bm{H}_k \Delta \bm{q}_{k+i}^{(2)}
\label{eq:error_term_expansion}
\end{align}

Simplifying, we get:
\begin{align}
\bm{E}_{k+i} =\, & \dfrac{1}{2} \Delta \bm{q}_{k+i}^{(1)\top} \bm{H}_k \Delta \bm{q}_{k+i}^{(1)} 
- \dfrac{1}{2} \Delta \bm{q}_{k+i}^{(1)\top} \bm{H}_k \Delta \bm{q}_{k+i}^{(2)} \nonumber \\
& + \dfrac{1}{2} \Delta \bm{q}_{k+i}^{(2)\top} \bm{H}_k \Delta \bm{q}_{k+i}^{(2)}.
\label{eq:error_term_simplified}
\end{align}

We can now bound the error $\bm{E}_{k+i}$.

\begin{lemma}
Under the assumption that the joint position changes $\Delta \bm{q}_{k+i}^{(1)}$ and $\Delta \bm{q}_{k+i}^{(2)}$ are small, and that the Hessian $\bm{H}_k$ is bounded, the error $\bm{E}_{k+i}$ is of order $\mathcal{O}\left( \| \Delta \bm{q} \|^2 \right)$.
\end{lemma}

\begin{proof}
Let $L_H$ be a bound such that $\| \bm{H}_k \| \leq L_H$. Then:
\begin{align}
\| \bm{E}_{k+i} \| \leq\, & \dfrac{1}{2} L_H \Big( 
\| \Delta \bm{q}_{k+i}^{(1)} \|^2 
+ \| \Delta \bm{q}_{k+i}^{(2)} \|^2 \nonumber \\
& \quad + \| \Delta \bm{q}_{k+i}^{(1)} \| 
\| \Delta \bm{q}_{k+i}^{(2)} \| \Big).
\label{eq:error_norm_bound}
\end{align}

Since $\| \Delta \bm{q}_{k+i}^{(1)} \|$ and $\| \Delta \bm{q}_{k+i}^{(2)} \|$ are small, their squares and products are of order $\mathcal{O}\left( \| \Delta \bm{q} \|^2 \right)$.
\end{proof}

Therefore, the error introduced by neglecting higher-order terms is negligible.

\begin{remark}[Magnitude of Higher-Order Terms]
\label{remark:higher_order_terms}
In the case of the modular manipulator, to quantify the smallness of the joint position increments, consider a control system running at \(100\)~Hz (i.e., with a \(\Delta t = 0.01\)~s control period) and assume a maximum joint velocity of \(2\)~rad/s. The maximum change in joint position over one control step is then
\(
\|\Delta \bm{q}_{k+i}\|_{\max} \approx 2 \text{ rad/s} \times 0.01 \text{ s} = 0.02 \text{ rad.}
\)
Since the neglected second-order terms in our analysis scale as
\(
\mathcal{O}\!\bigl(\|\Delta \bm{q}_{k+i}\|^2\bigr),
\)
their magnitude is on the order of \(0.02^2 = 0.0004\)~rad\(^2\). Consequently, the additional error introduced by ignoring these higher-order terms is minor and does not significantly affect the control or the system's stability and performance.
\end{remark}

\begin{theorem}
The linearized kinematic model in the low-level MPC implicitly includes second-order terms from the Taylor expansion without the need for explicit calculation of the second-order derivatives.
\end{theorem}

\begin{proof}
By updating the Jacobian $\mathbf{J}_{k+i}$ using the predicted joint positions $\bm{q}_{k+i}^{(1)}$ from the high-level MPC, the low-level MPC model includes the term $\left( \left. \dfrac{\partial \mathbf{J}}{\partial \bm{q}} \right|_{\bm{q}_k} \Delta \bm{q}_{k+i}^{(1)} \right) \Delta \bm{q}_{k+i}^{(2)}$, which captures part of the second-order effects in the end-effector position. The remaining error $\bm{E}_{k+i}$ is of order $\mathcal{O}\left( \| \Delta \bm{q} \|^2 \right)$ and is negligible under the assumption that joint position changes are small. Therefore, the low-level MPC implicitly includes second-order terms without explicitly computing the second derivatives.
\end{proof}

This approach enhances the fidelity of the kinematic model without significant computational overhead, as the low-level MPC retains its QP formulation by avoiding the explicit calculation of second derivatives in the approximated predictive kinematic model. The proof related to \textit{Linearized Acceleration Mapping Term} can obtained in Appendix.

\section{Evaluation}
To demonstrate that the controller is adaptable to various manipulator morphologies, we selected $5$ manipulators, each with distinct morphologies, to execute the same task under different scenarios. The DoF of these manipulators ranged from $4$ DoFs to $6$ DoFs, as shown in Fig.~\ref{fig:morphology_common}, with morphologies labeled from A to E. These manipulators feature different topological structures, specifically in the order of joint modular connections. The joint position limits for each module range from \(-2.75\) to \(2.75\) radians, with a maximum joint velocity of \(2.0 \, \text{rad/s}\) and a maximum joint acceleration of \(0.5 \, \text{m/s}^2\). 
The QP is solved using the off-the-shelf solver qpOASES~\cite{ferreau2014qpoases}.

\begin{figure}[h]
  \centering
\includegraphics[width=0.49\textwidth]{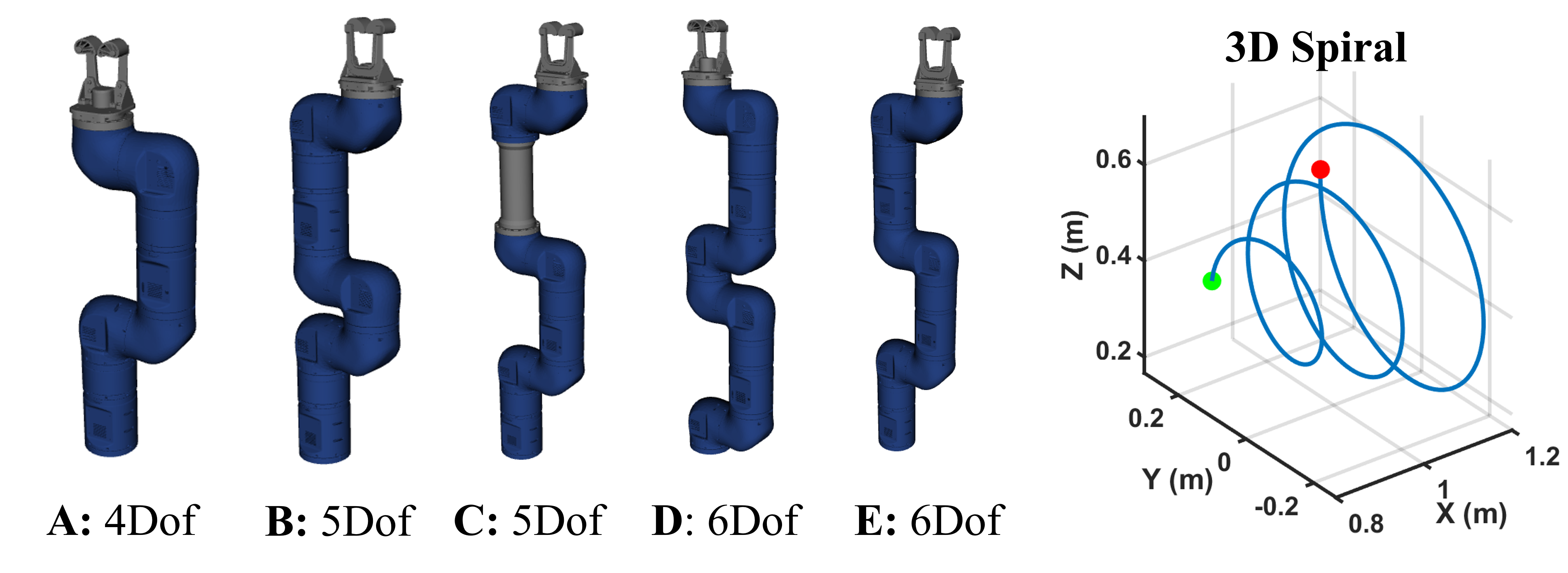}
  \vspace{-0.3cm}
  \caption{Morphology A to E. Left: Manipulator morphology. Right: Trajectory visualization. Green and red points represent the initial and final points of the trajectory.}
\label{fig:morphology_common}
  \vspace{-0.5cm} 
\end{figure}

\subsection{Comparison between MPC, HQP and HMPC}
In this section, we compare the execution performance of H-MPC, H-QP, and weight-based hierarchical MPC. H-QP optimizes the manipulator's joint states for a single step with a prediction horizon of N =1, without considering future trajectories. In contrast, weight-based hierarchical MPC (referred to as MPC here), as proposed by \cite{minniti2019whole}, manages tasks hierarchically by assigning higher gains to higher-priority tasks and lower gains to secondary tasks. Unlike H-MPC and H-QP, MPC combines all tasks into a single optimization problem with fixed weights based on task priority and does not dynamically update the predictive model in the low-level MPC. To evaluate these methods, we conducted a comparative study using five different manipulators, as illustrated in Fig.~\ref{fig:morphology_common}. The task involved executing a three-dimensional spiral trajectory, shown in Fig.~\ref{fig:morphology_common} (right). In the simulation, successful trajectory completion was not guaranteed due to potential singular configurations and joint limitations, which posed challenges for accurate tracking. In this task, position tracking was prioritized as the highest objective, while orientation tracking was secondary. 

\begin{figure}[h]
  \centering
\includegraphics[width=0.5\textwidth]{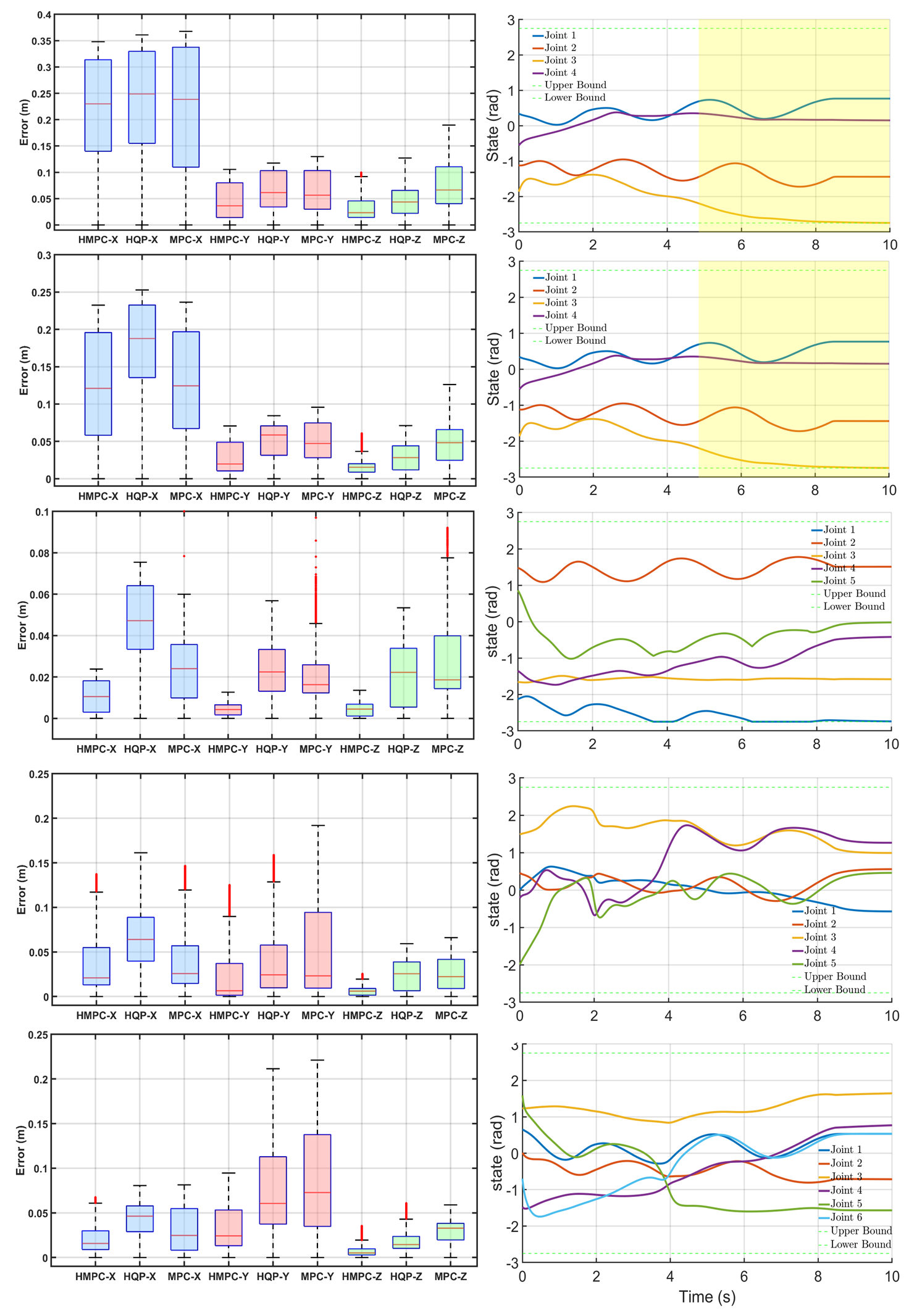}
  \caption{Left: box figure for tracking performance. Right: joint state for executing the task. 
  }
  \label{fig:box_plot}  
  \vspace{-0.5cm}
\end{figure}
\textbf{Performance Analysis}: 
We first analyze the robustness of tracking performance by comparing different control methods. To evaluate tracking performance, we recorded tracking errors in various manipulator morphologies by comparing the reference trajectory in the task space with the actual control tracking performance. The position-related tracking error in each control loop is calculated as \(\bm{e}_p = |\bm{p} - \bm{p}_{ref}|\), and the orientation tracking error is expressed as \(\bm{e}_o = | \log(\mathbf{R}\mathbf{R}_{ref}^T)^{\vee} | \), where \(\bm{p}\) denotes the current position of the manipulator's end-effector, \(\mathbf{R}\) represents the current orientation described by the rotation matrix, and \( (.)^{\vee}\) denotes the vectorization of the skew-symmetric part of a matrix. 
%
The box plots was plotted for comparing the tracking error of the high-level sub-tasks, while  joint state trajectory was plotted for the execution process of HMPC as show in the Fig.~\ref{fig:box_plot}. In the figure, the box plots reveal that H-MPC outperforms both MPC and HQP across all axes of the high-priority task in terms of median error. H-MPC consistently shows lower median errors, indicating closer adherence to desired outcomes with minimal deviations. Additionally, its narrower interquartile range (IQR) reflects a more concentrated and less variable error distribution. H-MPC also manages outliers more effectively, featuring fewer outliers and shorter whiskers, which highlights its robustness and ability to maintain tighter error bounds. Compared H-MPC and HQP, HMPC is outperformance in tracking trajectory as HMPC predicting future states and optimizing state trajectories that consider long-term objectives, rather than just the immediate ones. Meanwhile, H-MPC outperforms MPC by utilizing a dynamic predictive model, whereas MPC relies on a single, static linearized model. Additionally, unlike MPC, H-MPC does not require manual tuning of controller parameters, enhancing its adaptability and ease of use. For further details, readers can refer to the video \url{https://youtu.be/y1g8lXnXLPs}.

The joint-state trajectories across various morphologies are depicted in Fig.~\ref{fig:box_plot}. The joint states controlled by the H-MPC remain safely confined within predefined boundaries, as indicated by the dotted lines. Furthermore, the results demonstrates notable robustness when the robot's configuration approaches singularity, particularly in the yellow-highlighted regions near the workspace boundaries. For further details, readers can refer to the video \url{https://youtu.be/z7EIButAy74}.

\subsection{Pick-and-Place Experiments}

To demonstrate the effectiveness of the proposed H-MPC, especially in managing multi-phase manipulation tasks with non-redundant robots for executing specific task trajectories, we applied this control strategy in a real-world pick-and-place scenario. The experimental setup featured two distinct 5-DoF non-redundant manipulators and involved two different grasping orientations for handling a liquid-filled object. We carried out two experiments using the PINO modular manipulators, as detailed in~\cite{romiti2021toward}. The experimental process and its outcomes are showcased in an accompanying video, which can be viewed at \url{https://youtu.be/CSVVkLLrgqs}.

\section{Conclusion}
This work proposes a control policy based on H-MPC designed for modular systems that can adapt to various morphologies without requiring control parameter tuning. The H-MPC divides the controller into two levels: a high-level MPC and a low-level MPC. To ensure real-time control of the manipulator, each MPC formulates its optimization problem as a QP problem. In the low-level MPC, the linearized model is further updated using predictive results from the high-level MPC, enabling the incorporation of second-order information while maintaining a linearized formulation. We conduct extensive comparisons and experiments to demonstrate the effectiveness of the proposed control policy.

\section*{APPENDIX} ~\label{linear_acc}
\textbf{Inclusion of Second-Order Terms in Linearized Acceleration Mapping}: 
The end-effector acceleration is given by:
\begin{equation}
    \ddot{\bm{x}}_e = \mathbf{J}(\bm{q})\, \ddot{\bm{q}} + \dot{\mathbf{J}}(\bm{q}, \dot{\bm{q}})\, \dot{\bm{q}},
    \label{eq:end_effector_acceleration}
\end{equation}
where the time derivative of the Jacobian is:
\begin{equation}
    \dot{\mathbf{J}}(\bm{q}, \dot{\bm{q}}) = \frac{\partial \mathbf{J}}{\partial \bm{q}}\, \dot{\bm{q}}.
\end{equation}
The goal is to show that by updating $\mathbf{J}_{k+i}$ and $\dot{\mathbf{J}}_{k+i}$ using the predicted joint positions $\bm{q}_{k+i}^{(1)}$ and velocities $\dot{\bm{q}}_{k+i}^{(1)}$ from the high-level MPC, the low-level MPC implicitly includes second-order terms involving $\Delta \bm{q}$, $\Delta \dot{\bm{q}}$, and their products.


We perform a Taylor expansion of $\mathbf{J}(\bm{q})$ around $\bm{q}_k$:
\begin{equation}
    \mathbf{J}_{k+i} \approx \mathbf{J}_k + \left. \frac{\partial \mathbf{J}}{\partial \bm{q}} \right|_{\bm{q}_k} \Delta \bm{q}_{k+i}^{(1)} + \mathcal{O}\left( \| \Delta \bm{q}_{k+i}^{(1)} \|^2 \right),
    \label{eq:jacobian_expansion_acceleration}
\end{equation}
where $\Delta \bm{q}_{k+i}^{(1)} = \bm{q}_{k+i}^{(1)} - \bm{q}_k$.

Similarly, we expand $\dot{\mathbf{J}}(\bm{q}, \dot{\bm{q}})$ around $\bm{q}_k$ and $\dot{\bm{q}}_k$:
\begin{equation}
\begin{split}
\dot{\mathbf{J}}_{k+i} \approx \; & \dot{\mathbf{J}}_k  + \left. \frac{\partial \dot{\mathbf{J}}}{\partial \bm{q}} \right|_{\bm{q}_k} \Delta \bm{q}_{k+i}^{(1)}  + \left. \frac{\partial \dot{\mathbf{J}}}{\partial \dot{\bm{q}}} \right|_{\bm{q}_k} \Delta \dot{\bm{q}}_{k+i}^{(1)} \\
& + \mathcal{O}\Bigl( \| \Delta \bm{q}_{k+i}^{(1)} \|^2, \| \Delta \dot{\bm{q}}_{k+i}^{(1)} \|^2 \Bigr).
\end{split}
\label{eq:jacobian_derivative_expansion}
\end{equation}


We assume that the total joint position, velocity, and acceleration changes are approximated by the sum of the predicted changes from the high and low-level MPCs:
\begin{align}
    \Delta \bm{q}_{k+i} &\approx \Delta \bm{q}_{k+i}^{(1)} + \Delta \bm{q}_{k+i}^{(2)}, \label{eq:delta_q_total} \\
    \Delta \dot{\bm{q}}_{k+i} &\approx \Delta \dot{\bm{q}}_{k+i}^{(1)} + \Delta \dot{\bm{q}}_{k+i}^{(2)}, \label{eq:delta_dq_total} \\
    \ddot{\bm{q}}_{k+i} &\approx \ddot{\bm{q}}_{k+i}^{(1)} + \ddot{\bm{q}}_{k+i}^{(2)}. \label{eq:ddq_total}
\end{align}


Using equations \eqref{eq:end_effector_acceleration}, \eqref{eq:jacobian_expansion_acceleration}, and \eqref{eq:jacobian_derivative_expansion}, the end-effector acceleration at time $k+i$ can be approximated as:
\begin{equation}
\begin{aligned}
    \ddot{\bm{x}}_{e,k+i} &\approx \left( \mathbf{J}_k + \left. \frac{\partial \mathbf{J}}{\partial \bm{q}} \right|_{\bm{q}_k} \Delta \bm{q}_{k+i}^{(1)} \right) \left( \ddot{\bm{q}}_{k+i}^{(1)} + \ddot{\bm{q}}_{k+i}^{(2)} \right) \\
    &\quad + \left( \dot{\mathbf{J}}_k + \left. \frac{\partial \dot{\mathbf{J}}}{\partial \bm{q}} \right|_{\bm{q}_k} \Delta \bm{q}_{k+i}^{(1)} + \left. \frac{\partial \dot{\mathbf{J}}}{\partial \dot{\bm{q}}} \right|_{\bm{q}_k} \Delta \dot{\bm{q}}_{k+i}^{(1)} \right) \\
    & \left( \dot{\bm{q}}_{k+i}^{(1)} + \dot{\bm{q}}_{k+i}^{(2)} \right) + \mathcal{O}\left( \| \Delta \bm{q} \|^2 \right).
\end{aligned}
\label{eq:end_effector_acceleration_expanded}
\end{equation}

The second-level MPC uses the following model for predicting the end-effector acceleration:
\begin{equation}
    \ddot{\bm{x}}_{e,k+i}^{(2)} = \mathbf{J}_{k+i} \ddot{\bm{q}}_{k+i}^{(2)} + \dot{\mathbf{J}}_{k+i} \dot{\bm{q}}_{k+i}^{(2)}.
    \label{eq:second_mpc_acceleration_model}
\end{equation}
Substituting the expansions \eqref{eq:jacobian_expansion_acceleration} and \eqref{eq:jacobian_derivative_expansion} into \eqref{eq:second_mpc_acceleration_model}, we have:
\begin{equation}
\begin{aligned}
    \ddot{\bm{x}}_{e,k+i}^{(2)} &\approx \left( \mathbf{J}_k + \left. \frac{\partial \mathbf{J}}{\partial \bm{q}} \right|_{\bm{q}_k} \Delta \bm{q}_{k+i}^{(1)} \right) \ddot{\bm{q}}_{k+i}^{(2)} \\
    &\quad + \left( \dot{\mathbf{J}}_k + \left. \frac{\partial \dot{\mathbf{J}}}{\partial \bm{q}} \right|_{\bm{q}_k} \Delta \bm{q}_{k+i}^{(1)} + \left. \frac{\partial \dot{\mathbf{J}}}{\partial \dot{\bm{q}}} \right|_{\bm{q}_k} \Delta \dot{\bm{q}}_{k+i}^{(1)} \right) \dot{\bm{q}}_{k+i}^{(2)}.
\end{aligned}
\label{eq:second_mpc_acceleration_expanded}
\end{equation}


We define the error $\bm{E}_{k+i}$ between the actual end-effector acceleration $\ddot{\bm{x}}_{e,k+i}$ and the combined prediction from the high and low-level MPCs:
\begin{equation}
    \bm{E}_{k+i} = \ddot{\bm{x}}_{e,k+i} - \left( \ddot{\bm{x}}_{e,k+i}^{(1)} + \ddot{\bm{x}}_{e,k+i}^{(2)} \right),
\end{equation}
where $\ddot{\bm{x}}_{e,k+i}^{(1)} = \mathbf{J}_{k+i} \ddot{\bm{q}}_{k+i}^{(1)} + \dot{\mathbf{J}}_{k+i} \dot{\bm{q}}_{k+i}^{(1)}$.

Substituting \eqref{eq:end_effector_acceleration_expanded}, \eqref{eq:second_mpc_acceleration_expanded}, and the expression for $\ddot{\bm{x}}_{e,k+i}^{(1)}$, we have:
\begin{align}
\dot{\mathbf{J}}_{k+i} \approx\, & \dot{\mathbf{J}}_k 
+ \left. \frac{\partial \dot{\mathbf{J}}}{\partial \bm{q}} \right|_{\bm{q}_k} \Delta \bm{q}_{k+i}^{(1)} \nonumber \\
& + \left. \frac{\partial \dot{\mathbf{J}}}{\partial \dot{\bm{q}}} \right|_{\bm{q}_k} \Delta \dot{\bm{q}}_{k+i}^{(1)} 
+ \mathcal{O} \left( \| \Delta \bm{q}_{k+i}^{(1)} \|^2, \| \Delta \dot{\bm{q}}_{k+i}^{(1)} \|^2 \right).
\label{eq:jacobian_derivative_expansion}
\end{align}

After simplification and neglecting higher-order terms, the error reduces to:
\begin{equation}
    \bm{E}_{k+i} \approx \mathcal{O}\left( \| \Delta \bm{q}_{k+i}^{(1)} \| \| \ddot{\bm{q}}_{k+i}^{(1)} \|,\, \| \Delta \bm{q}_{k+i}^{(1)} \| \| \dot{\bm{q}}_{k+i}^{(1)} \| \right).
\end{equation}

\begin{lemma}
Under the assumption that the joint changes $\Delta \bm{q}_{k+i}^{(1)}$, $\Delta \dot{\bm{q}}_{k+i}^{(1)}$, $\ddot{\bm{q}}_{k+i}^{(1)}$, and $\dot{\bm{q}}_{k+i}^{(1)}$ are small, and that the derivatives of the Jacobian are bounded, the error $\bm{E}_{k+i}$ is of order $\mathcal{O}\left( \| \Delta \bm{q} \|^2 \right)$.
\end{lemma}

\begin{proof}
Let $L_J$, $L_{\dot{J}}$ be bounds such that:
\begin{equation}
    \left\| \frac{\partial \mathbf{J}}{\partial \bm{q}} \right\| \leq L_J, \quad \left\| \frac{\partial \dot{\mathbf{J}}}{\partial \bm{q}} \right\| \leq L_{\dot{J}}, \quad \left\| \frac{\partial \dot{\mathbf{J}}}{\partial \dot{\bm{q}}} \right\| \leq L_{\dot{J}}.
\end{equation}

Then the error norm satisfies:
\begin{equation}
    \| \bm{E}_{k+i} \| \leq C \left( \| \Delta \bm{q}_{k+i}^{(1)} \| \| \ddot{\bm{q}}_{k+i}^{(1)} \| + \| \Delta \bm{q}_{k+i}^{(1)} \| \| \dot{\bm{q}}_{k+i}^{(1)} \| \right),
\end{equation}
where $C$ is a constant depending on $L_J$ and $L_{\dot{J}}$.

Since $\Delta \bm{q}_{k+i}^{(1)}$, $\ddot{\bm{q}}_{k+i}^{(1)}$, and $\dot{\bm{q}}_{k+i}^{(1)}$ are small due to high control update rates, their products are of order $\mathcal{O}\left( \| \Delta \bm{q} \|^2 \right)$.
\end{proof}


By updating $\mathbf{J}_{k+i}$ and $\dot{\mathbf{J}}_{k+i}$ using $\bm{q}_{k+i}^{(1)}$ and $\dot{\bm{q}}_{k+i}^{(1)}$, the second-level MPC includes second-order terms involving $\Delta \bm{q}_{k+i}^{(1)}$, $\Delta \dot{\bm{q}}_{k+i}^{(1)}$, $\ddot{\bm{q}}_{k+i}^{(2)}$, and $\dot{\bm{q}}_{k+i}^{(2)}$. The error introduced by neglecting higher-order terms is negligible, being of order $\mathcal{O}\left( \| \Delta \bm{q} \|^2 \right)$. These second-order terms capture the curvature and dynamic changes in the manipulator's kinematics, improving the fidelity of the acceleration mapping without the explicit calculation of second derivatives.

In summary, the implicit inclusion of second-order terms through the updates of the Jacobian $\mathbf{J}$ and its derivative $\dot{\mathbf{J}}$ enhances the accuracy of the low-level MPC's control model, while keeping the computational complexity manageable.


\bibliographystyle{ieeetr}
\bibliography{reference}

\end{document}